\documentclass[letterpaper, 10 pt, conference]{ieeeconf}  

\IEEEoverridecommandlockouts                              

\usepackage[colorlinks,linkcolor=black,anchorcolor=black,urlcolor=blue]{hyperref}
\usepackage{times}
\usepackage[pdftex]{graphicx}
\usepackage{wrapfig}
\newcommand{\lmr}[1]{{#1}}
\usepackage{tabularx}
\usepackage{amsmath,amsfonts,amssymb}
\usepackage{prettyref}
\usepackage{mathrsfs}
\usepackage{graphicx}
\usepackage{wrapfig}
\usepackage{MnSymbol}
\usepackage{multirow}
\usepackage{booktabs}
\usepackage{algorithm}
\usepackage[noend]{algpseudocode}
\usepackage[table]{xcolor}
\usepackage{cellspace}

\newtheorem{theorem}{Theorem}[section]
\newtheorem{lemma}[theorem]{\TE{Lemma}}

\algnewcommand{\LineComment}[1]{\State \(\triangleright\) #1}
\algdef{SE}[DOWHILE]{Do}{doWhile}{\algorithmicdo}[1]{\algorithmicwhile\ #1}
\newcommand*{\colorboxed}{}
\def\colorboxed#1#{%
  \colorboxedAux{#1}%
}
\newcommand*{\colorboxedAux}[3]{%
  \begingroup
    \colorlet{cb@saved}{.}%
    \color#1{#2}%
    \boxed{%
      \color{cb@saved}%
      #3%
    }%
  \endgroup
}

\newrefformat{Fig}{Figure~\ref{#1}}
\newrefformat{fig}{Figure~\ref{#1}}
\newrefformat{par}{Section~\ref{#1}}
\newrefformat{appen}{Appendix~\ref{#1}}
\newrefformat{sec}{Section~\ref{#1}}
\newrefformat{sub}{Section~\ref{#1}}
\newrefformat{table}{Table~\ref{#1}}
\newrefformat{alg}{Algorithm~\ref{#1}}
\newrefformat{Alg}{Algorithm~\ref{#1}}
\newrefformat{Def}{Definition~\ref{#1}}
\newrefformat{Thm}{Theorem~\ref{#1}}
\newrefformat{Lem}{Lemma~\ref{#1}}
\newrefformat{step}{Step~\ref{#1}}
\newrefformat{ln}{Line~\ref{#1}}
\newrefformat{eq}{Equation~\ref{#1}}
\newrefformat{pb}{Problem~\ref{#1}}
\newrefformat{it}{Item~\ref{#1}}
\newrefformat{te}{Term~\ref{#1}}
\def\Eqref Eq:#1:{\eqref{eq:#1}}
\newrefformat{Eq}{Equation~\Eqref#1:}

\newcommand{\E}[1]{\mathbf{#1}}
\newcommand{\TE}[1]{\textbf{#1}}

\newcommand{\THREEC}[3]{\left(\setlength{\arraycolsep}{1pt}\begin{array}{c}#1 \\ #2 \\ #3\end{array}\right)}

\newcommand{\FIVE}[5]{\left(\setlength{\arraycolsep}{1pt}\begin{array}{ccccc}{#1}, & {#2}, & {#3}, & {#4}, & {#5}\end{array}\right)}

\newcommand{\argmin}[1]{\underset{#1}{\E{argmin}}}
\newcommand{\argminP}[1]{\E{argmin}}

\newcommand{\argmaxP}[1]{\E{argmax}}

\definecolor{darkgreen}{HTML}{186a3b}
\newcommand{\BBNode}{\text{BBNode}}
\newcommand{\SOSONE}{\text{SOS}_1}
\newcommand{\SOSTWO}{\text{SOS}_2}
\newcommand{\SOTHREE}{\text{SO}_3}
\newcommand{\PBALL}{B}

\newcommand{\MIP}{\text{MICP}}
\definecolor{COMMENT_COLOR}{RGB}{52, 156, 219}

\title{\LARGE \bf
New Formulation of Mixed-Integer Conic 
Programming for Globally Optimal Grasp Planning
}

\author{Min Liu$^{1,3}$, Zherong Pan$^{2}$, Kai Xu$^{1*}$, and Dinesh Manocha$^{3}$ \\
{\href{https://gamma.umd.edu/researchdirections/grasping/global_grasp_planner}{https://gamma.umd.edu/researchdirections/grasping/global\_grasp\_planner}}
\thanks{$^{1}$Min Liu  and Kai Xu are with School of Computer, National University of Defense Technology, Changsha, HN 410073, China 
        {\tt\small gfsliumin@gmail.com;} 
        {\tt\small kevin.kai.xu@gmail.com}}%
        
\thanks{$^{2}$Zherong Pan is with Department of Computer Science, University of North Carolina at Chapel Hill, Chapel Hill, NC 27514, USA
        {\tt\small zherong@cs.unc.edu}}%
        
\thanks{$^{3}$Min Liu and Dinesh Manocha are with Department of Computer Science and Electrical \& Computer Engineering, University of Maryland at College Park, College Park, MD, 20740 USA 
        {\tt\small dm@cs.umd.edu}}%
        
\thanks{$^{*}$Kai Xu is the corresponding author.}%
}

\begin{document}

\maketitle
\thispagestyle{empty}
\pagestyle{empty}

\begin{abstract}
We present a two-level branch-and-bound (BB) algorithm to compute the optimal gripper pose that maximizes a grasp metric \lmr{in a restricted search space}. Our method can take the gripper's kinematics feasibility into consideration to ensure that a given gripper can reach the set of grasp points without collisions or predict infeasibility with finite-time termination when no pose exists for a given set of grasp points. Our main technical contribution is a novel mixed-integer conic programming (MICP) formulation for the inverse kinematics of the gripper that uses a small number of binary variables and tightened constraints, which can be efficiently solved via a low-level BB algorithm. Our experiments show that optimal gripper poses for various target objects can be computed taking 20-180 minutes of computation on a desktop machine and the computed grasp quality, in terms of the $\E{Q}_1$ metric, is better than those generated using sampling-based planners.
\end{abstract}


\section{INTRODUCTION}
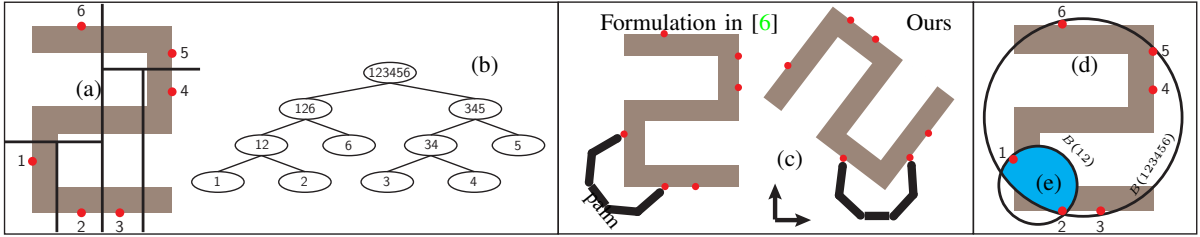
\begin{figure*}[th]
\centering
\scalebox{0.85}{
\includegraphics[width=0.99\textwidth]{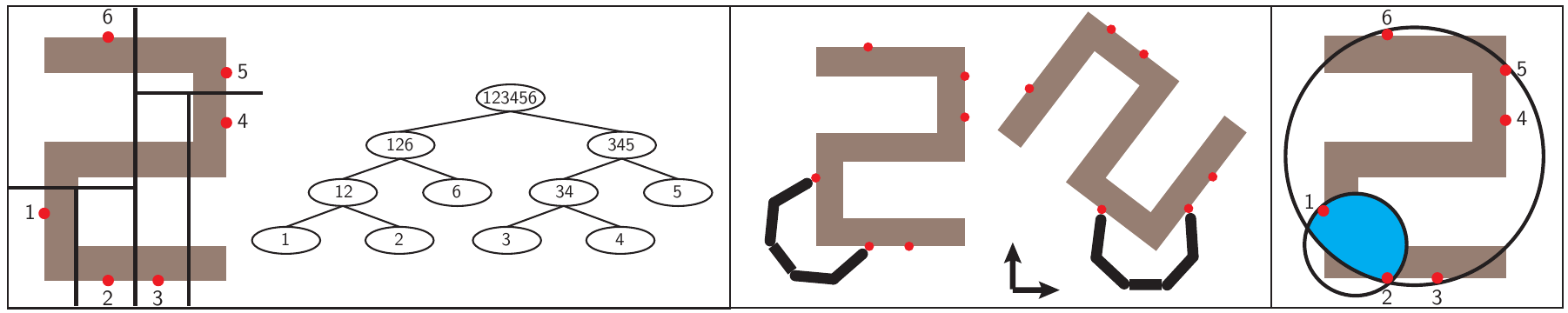}
\put(-478,60){(a)}
\put(-310,70){(b)}
\put(-180,30){(c)}
\put(-263,17){\rotatebox{-51}{\small{palm}}}
\put(-255,88){Formulation in \cite{dai2018synthesis}}
\put(-125,88){Ours}
\put(-55 ,70){(d)}
\put(-70 ,20){(e)}
\put(-31 ,17){\rotatebox{57}{\tiny$B(123456)$}}
\put(-60 ,42){\rotatebox{-45}{\tiny$B(12)$}}}
\vspace{-10px}
\caption{\label{fig:pipeline} \small{(a): A toy example where the target object is a Z-shape (olive) on which we have $P=6$ potential grasp points (red point). (b): We first build a KD-tree for the potential grasp points. Our gripper has $L=5$ links and $K=2$ fingertip points. (c): For any $\BBNode$, e.g., $\BBNode(1,2)$, we will use the IK solver for a feasibility check. We allow MICP to use few binary variables by allowing the target object to move and fixing the palm of the gripper. (d): We build a bounding sphere, i.e. $\PBALL(123456)$ and $\PBALL(12)$, for each non-leaf KD-tree node. (e): For an internal $\BBNode(12,\bullet)$ ($\bullet$ means any KD-tree node), $x_1$ can either be at $p_1$ or $p_2$. This constraint has a convex relaxation that requires $x_1$ to be inside the intersection of bounding spheres (blue): $\PBALL(123456)\cap\PBALL(126)\cap\PBALL(12)$.}}
\vspace{-10px}
\end{figure*}
Grasp planning is a well-studied problem in robotics and there is a large amount of work in grasp metric computation \cite{roa2015grasp} and gripper pose planning \cite{ciocarlie2007dexterous}. Since the two components are somewhat independent, practitioners can build versatile planning frameworks that allow an arbitrary combination of grasp metrics and gripper pose planners for different applications \cite{1371616}. A high number of choices have been proposed for grasp metrics \cite{roa2015grasp}, and a few gripper pose planners are also known. Some planners such as \cite{dai2018synthesis,1220716} return sub-optimal solutions, which are sensitive to initial guesses and can return grasps of low qualities. Another planner based on simulated annealing (SA) was proposed in \cite{ciocarlie2007dexterous}, which can compute the optimal solution if an infinite number of samples is allowed.

A promising direction of previous works \cite{7812687,schulman2017grasping} use branch-and-bound (BB) to compute optimal grasp points that maximize a given grasp metric. Unlike SA, BB returns the optimal solution or predicts infeasibility. However, BB algorithms in \cite{7812687,schulman2017grasping} only consider the optimality in grasp points, the kinematics feasibility of gripper is either omitted in \cite{schulman2017grasping} or considered without optimality guarantee in \cite{7812687}.  SCORES\cite{zhu2018scores} uses a recursive neural network for shape composition, and the representation of shape composition can be used for 3D object 
grasp planning; Triangle lasso~\cite{zhao2018triangle} can be used for simultaneous clustering and optimization, and is potential for  grasp optimization. To take the gripper's kinematics into consideration, an inverse kinematics (IK) algorithm is needed to determine whether a given set of grasp points can be reached by the gripper. However, most available inverse kinematics algorithms, such as \cite{murray2017mathematical,doi:10.1177/0278364910396389}, are not optimal and can miss feasible solutions when one exists. Recently, a complete IK algorithm is presented in \cite{dai2017global}, which reformulates IK as a mixed-integer conic programming (MICP). However, \cite{dai2017global} involves the use of a large number of integer parameters making it slow to solve \lmr{because the worst-case complexity of MICP is exponential in the number of integer variables.}

\TE{Main Results:} We present a novel, two-level BB algorithm to compute the optimal gripper pose that maximizes a given grasp metric \lmr{in a restricted search space. Our high-level BB algorithm searches for points on the object's surface that can potentially be used as contact points. Our low-level BB algorithm searches for collision-free gripper poses that realize the given set of contact points. A set of lazy-evaluation heuristic techniques are used to remove unnecessary searches and reduce the branch factor.} 
We have tested our algorithm on 10 target objects grasped by a 3-finger gripper with 15 DOFs and a \lmr{Barrett Hand} with \lmr{10} DOFs. Our experiments show that optimal grasps can be computed within 20-180 minutes on a desktop machine for different grippers. Furthermore, our low-level BB formulation results in a speedup of $100\times$ over \cite{dai2017global} in terms of gripper's kinematics feasibility check. We have also compared our algorithm's performance with a sample-based grasp planner \cite{ciocarlie2007dexterous} and observed the following benefits:
\begin{itemize}
    \item Our method always computes higher quality grasps based on $Q_1$ metric, though we are $6-10\times$ slower.
    \item Our method can detect infeasibility within finite time, which happens frequently when target objects are large compared with the gripper.
\end{itemize}

\section{RELATED WORK\label{sec:related}}
We review previous works on grasp metric computation, gripper pose planning, and IK algorithms.

\TE{\lmr{Grasp Planning}} takes the gripper's kinematics feasibility into consideration, which computes a \lmr{gripper} pose given the set of contact points as end-effector constraints. Some sampling-based planners \cite{ciocarlie2007dexterous,1371616} determine the gripper's pose first by sampling in the gripper's configuration space. \lmr{Varava \cite{Varava_caging} presented an algorithm that can check whether a geometric body can cage another one or detect infeasibility.} However, it is rather difficult for the fingers to exactly lie on the surface of target objects, so these planners have to close the gripper to have the fingers on the object surface. Other planners, such as \cite{7812687} and our method, first select contact points, compute the grasp quality, and then solve the IK problem to compute the gripper's pose.

\TE{Grasp Metrics} measure the quality of a grasp and provide ways to compare different grasps. A dozen different grasp metrics have been proposed and summarized in \cite{roa2015grasp}. Sampling-based planners can be used to optimize all kinds of grasp metrics. However, BB can only be used when a grasp metric is monotonic \cite{7812687,schulman2017grasping}, i.e. the grasp metric value computed for a superset of grasp points must be larger than \lmr{or equal to} that computed for a subset. Fortunately, most metrics, including $\E{Q}_{VEW}$ \cite{769}, $\E{Q}_{1,\infty}$ \cite{219918}, are monotonic. \lmr{We use the $\E{Q}_1$ metric in our algorithm. The $\E{Q}_1$ metric assumes that the sum of the magnitude of forces is no greater than 1. Every contact will generate a wrench and the $\E{Q}_1$ metric value is equal to the residual radius of the convex hull of all these generated wrenches. Unlike \cite{dai2018synthesis}, our frictional cone is quadratic, i.e. no further linearizations are used. Obviously, a positive $\E{Q}_1$ metric implies force-closure.} In particular, \cite{dai2018synthesis} showed that the computation of $\E{Q}_1$ can be approximated by solving a semidefinite programming (SDP), allowing the BB to be solved using off-the-shelf mixed-integer SDP solvers \cite{doi:10.1080/10556788.2017.1322081}. Instead, we propose to use a more generic form of BB-based algorithm for our high-level problem in order to account for many different metric types.

\TE{IK Algorithms} can be used to check kinematics feasibility. However, these algorithms are sub-optimal or sampling-based. A sub-optimal IK algorithm such as \cite{article,murray2017mathematical,doi:10.1177/0278364910396389} can return false negatives, i.e. reporting infeasibility when a solution exists. On the other hand, a sampling-based IK algorithm such as \cite{doi:10.1177/0278364910396389} can always find the solution but assume an infinite amount of samples are used. Recently, a new IK algorithm based on MICP relaxation is proposed in \cite{dai2018synthesis}, which finds a solution or detects infeasibility in a finite amount of time. However, we found that the formulation of \cite{dai2018synthesis} requires too many binary variables, making MICP solve time-consuming 
for its combinatorial worst-case complexity.
\section{PROBLEM STATEMENT\label{sec:problem}}
In this section, we formulate the problem of grasp planning. All the symbols used in this paper are summarized in \prettyref{table:symbols}. As shown in \prettyref{fig:pipeline}, the planner input includes:
\begin{itemize}
    \item A target object that occupies a volume $\Omega_o\subset R^3$.
    \item A set of $P$ potential grasp points: $p_{1,\cdots,P}\in\partial\Omega_o$.
    \item A gripper represented as an articulated object, i.e. a set of $L$ rigid links. Each link occupies a volume $\Omega_i(\theta)\subset R^3$, where $i=1,\cdots,L$ and $\theta$ is the set of joint angle and globally transformation parameters. On the gripper, there is a set of $K$ fingertip points: $x_{1,\cdots,K}(\theta)$ and $K<L$. \lmr{Without loss of generality}, we always assume the first $K$ links are fingertip links so that $x_i\in\Omega_i$.
    \item A grasp metric $Q(X)$ whose input is a set $X$ of grasp points satisfying $\forall x\in X, x\in\partial\Omega_o$. Moreover, we assume that the grasp metric is monotonic, i.e. $A\subseteq B\implies Q(A)\leq Q(B)$.
\end{itemize}
In this paper, we assume that the first $6$ parameters in $\theta$ are extrinsic parameters ($3$ for rotation and $3$ for translation) and the rest are intrinsic parameters, i.e. joint angles. Given these inputs, the planner either predicts that the problem is infeasible or outputs $\theta^*$ satisfying the following conditions:
\begin{itemize}
    \item \TE{C1:} The gripper does not collide with the target object or have self-collisions. In other words, $\forall i=1,\cdots,L, \Omega_o\cap\Omega_i(\theta^*)=\emptyset$ and $\forall 1\leq i<j\leq L, \Omega_i(\theta^*)\cap\Omega_j(\theta^*)=\emptyset$.
    \item \TE{C2:} Each fingertip point lies on the object surface, i.e. $x_{i,\cdots,K}(\theta^*)\in\{p_{1,\cdots,P}\}$.
    \item \TE{C3:} For all other $\theta$ satisfying \TE{C1} and \TE{C2}, we have: $Q(\{x_{i,\cdots,K}(\theta)\})\leq Q(\{x_{i,\cdots,K}(\theta^*)\})$.
\end{itemize}
\begin{table}
\centering
\caption{\label{table:symbols} \small{Symbol Table}}
\vspace{-10px}
\scalebox{0.7}{
\rowcolors{0}{gray!50}{white}
\setlength{\tabcolsep}{1pt}
\begin{tabular}{llll}
\toprule
Variable & Definition &  Variable & Definition \\
\midrule
$\Omega_o$ & target object & $\Omega_i$ & $i$th link of gripper \\
\lmr{$\partial\Omega_o$} & \lmr{surface of target object} & $p_i$ & $i$th potential grasp point \\
$x_i$ & $i$th fingertip point & $n(p_i)$ & normal on $p_i$  \\
$n(x_i)$ & normal on $x_i$  & $\lambda/\beta/\gamma$ & auxiliary variables \\
$X$ & A set of potential grasp points &$X^i$ & The KD-tree node for $x_i$ \\
$X_p/X_R$ & parent of KD-tree node $X$ / root of tree & $X_l/X_r$ & left / right child of KD-tree node $X$ \\
$\theta$ & gripper's kinematics parameter  & $d_i$ & grid index in piecewise approximation   \\
$\theta_i$ & kinematics parameter influencing $x_i$  & $l_i^j/u_i^j$ & $j$th lower / upper bound in $\theta_i$   \\
$\Theta$ & conceptual solution space & $B/C$ & minimal bounding sphere / cone   \\
$c/r$ & center / radius of $B$  & $m/\epsilon$ & center / radius of $C$ \\
$\epsilon'$ & $\epsilon$ considering user threshold &$L$ & number of gripper links \\
$K$ & number of fingertip points & $P$ & number of potential grasp points \\
$S$ & number of separating directions &$Q$ & monotonic grasp metric \\
$R_i/t_i$ & global rotation / translation of $\Omega_i$  &$R/t$ & global rotation / translation of $\Omega_o$  \\
$N$ & \#cells in piecewise approximation & $s_k$ & $k$th separating direction  \\
$w$ & rotation vector of $R$  & $D$ & penetration depth \\
\hline
\end{tabular}}
\vspace{-10px}
\end{table}
Note that previous works \cite{7812687,schulman2017grasping} ignore \TE{C1} and \TE{C2} and solve the problem using a one-level BB algorithm. On the other hand, the sampling-based planner \cite{ciocarlie2007dexterous} solves the full version of this problem by generating samples of $\theta$ and then testing \TE{C1}, \TE{C2}, and \TE{C3}, but SA cannot detect infeasibility. Instead, our two-level BB algorithm takes \TE{C1}, \TE{C2}, and \TE{C3} into consideration with finite time termination. Conceptually, we identify a large enough subset $\Theta$ of the entire solution space. If we restrict ourselves by adding a condition, \TE{C4:} $\theta\in\Theta$, then the optimal solution to the planning problem can be efficiently solved within a finite amount of computational time. \lmr{We solve for global optimal solutions in a restricted search space because we only sample finite potential grasp points for \TE{C1} and we use a subset of the solution space for a gripper's kinematics parameters for \TE{C4}.}

\section{TWO-LEVEL BRANCH-AND-BOUND FORMULATION\label{sec:formulation}}
\subsection{\lmr{BB Algorithm}}
\lmr{BB algorithms can efficiently find globally optimal solutions for non-convex optimization problems \cite{Morrison:2016:BA:2899521.2899589} in the form of disjoint convex sets, which means that an optimization problem can be decomposed into several sub-cases where each case is convex. A BB algorithm can efficiently prune sub-optimal cases at an early stage and accelerate the computation. To this end, a search tree is constructed, each node of which corresponds to a relaxed convex problem. Starting from the root node, each node is evaluated to either find a solution or to prove infeasibility or sub-optimality. If a node is infeasible or its solution is sub-optimal, all its child nodes must also be infeasible or sub-optimal and they will be excluded from further traversal. Otherwise, the node is branched into two or more child nodes. The key to the success of a BB algorithm is the design of the relaxed convex problem. In our high-level BB, the relaxation is provided by the monotonicity of the grasp metric. In our low-level BB, the relaxation is provided by turning all the integer variables into continuous variables.}
\subsection{High-Level BB}
Our high-level BB takes a very similar form as \cite{7812687}. We select $K$ points, $x_{1,\cdots,K}$, from the set of $P$ potential grasp points, $\{p_{1,\cdots,P}\}$, such that the grasp quality $Q(\{x_{1,\cdots,K}\})$ is maximized. To solve this problem, we first build a KD-tree for $\{p_{1,\cdots,P}\}$. As illustrated in \prettyref{fig:pipeline}b, each KD-tree node is uniquely denoted by a subset $X\subset\{p_{1,\cdots,P}\}$. The KD-tree is used both by our high-level and low-level BB. A balanced KD-tree can effectively reduce the length of search path in high-level BB. It can also restrict the search space and accelerate MICP solve in low-level BB.

The BB algorithm builds a search tree and keeping track of the best solution with the largest grasp quality metric found so far, which is defined as $Q_{best}$. Each node on the search tree can be uniquely denoted by $\BBNode(X^1,\cdots,X^K)$, where each $X^i$ is the KD-tree node for the $i$th fingertip point. This $X^i$ is also the set of potential grasp points that $x_i$ can possible be at. In other words, each $\BBNode$ represents a Cartesian product of the $K$ set of potential graph points. At each $\BBNode$, we encounter one of the two cases:
\begin{itemize}
    \item If $|X^i|=1$ for all $i$, then the $\BBNode$ is a leaf node and we compute tentative grasp quality for this node: $Q(X^1\cup X^2\cdots\cup X^K)$. If the tentative grasp quality is larger than $Q_{best}$, then this $\BBNode$ is known as an incumbent and $Q_{best}$ is updated.
    \item  If there is an $i$ such that $|X^i|>1$, then the $\BBNode$ is a non-leaf node. In this case, we also compute the tentative grasp quality, $Q_{upper}=Q(X^1\cup X^2\cdots\cup X^K)$, for this node. If the tentative grasp quality is smaller than $Q_{best}$, i.e. $Q_{upper}<Q_{best}$, then this $\BBNode$ is eliminated for further processing. Otherwise, we branch on all the $X^i$ with $|X^i|>1$.
\end{itemize}
It has been shown in \cite{7812687,schulman2017grasping} that this algorithm will find the optimal $\{x_{1,\cdots,K}\}$ if $Q$ is monotonic. When $Q$ is monotonic, the tentative grasp quality $Q_{upper}$ is an lifting of the grasp quality metric to a superset, which is also an upper bound of the actual grasp quality. Therefore, rejecting $\BBNode$ when $Q_{upper}<Q_{best}$ will not miss better solutions. However, our high-level BB does not take the gripper's kinematics feasibility into consideration. Each $\BBNode$ essentially specifies all the possible positions of each fingertip point. If it is impossible for the gripper to reach these positions, then the given $\BBNode$ does not contain feasible solutions and should be cut early to avoid the redundant search.

\subsection{Gripper's Inverse Kinematics\label{sec:IK}}
Before we discuss feasibility checks of $\BBNode$, we first propose a novel, MICP-based optimal IK algorithm, which is the cornerstone of our feasibility check algorithm. Compared with \cite{dai2018synthesis}, which can be applied to solve IK for any articulated robot, our formulation only works for the problem of gripper pose planning but uses much fewer binary variables, leading to significant speedup. 

As illustrated in \prettyref{fig:pipeline}c, our main idea is that applying a global transformation of the gripper is equivalent to applying a global inverse transformation of the target object while keeping the palm of gripper fixed. However, if we keep the palm of gripper fixed, then the fingers of the gripper become decoupled. Specifically, we assume that each fingertip $x_i(\theta)=x_i(\theta_i)$ such that: $\theta=\FIVE{0,0,0,0,0,0}{\theta_1}{\theta_2}{\cdots}{\theta_K}$. This assumption holds if we allow the target object to have a global rigid transformation.

Based on this assumption, we can relax the IK problem as MICP. Specifically, we introduce auxiliary variables $R_i,t_i$ for the rotation and translation of the $i$th link. The main constraint to relax is $R_i\in\SOTHREE$ where $R_i$ is also a function of $\theta_i$. We relax $R_i(\theta_i)$ using a piecewise linear approximation by introducing the following constraints:
\begin{small}
\begin{equation}
\begin{aligned}
\label{eq:XChoose}
&R_i=\sum_{d_1,\cdots,d_{|\theta_i|}=0}^N 
\lambda_i^{d_1,\cdots,d_{|\theta_i|}} 
R_i(\theta_i^{d_1,\cdots,d_{|\theta_i|}}) \\[-3px]
&\sum_{d_1,\cdots,d_{|\theta_i|}=0}^N 
\lambda_i^{d_1,\cdots,d_{|\theta_i|}}=1    \\[-3px]
&\sum_{d_1,\cdots,d_{j-1},d_{j+1},\cdots,d_{|\theta_i|}=0}^N 
\lambda_i^{d_1,\cdots,d_{|\theta_i|}}\in\SOSTWO \;\forall j=1,\cdots,|\theta_i|,
\end{aligned}
\end{equation}
\end{small}
where $\SOSTWO$ is the special ordered set of type 2 \cite{vielma2011modeling} and $\lambda_i^{d_1,\cdots,d_{|\theta_i|}}$ are continuous-valued auxiliary variables. This piecewise linear approximation restricts the solution space, which corresponds to our last condition \TE{C4} in \prettyref{sec:problem}. The mixed-integer constraints in \prettyref{eq:XChoose} require $|\theta_i|$ $\SOSTWO$ constraints and hence $|\theta_i|\lceil\E{log}_2N\rceil$ binary decision variables. Finally, $\theta_i^{d_1,\cdots,d_{|\theta_i|}}$ is defined as:
\begin{align}
\theta_i^{d_1,\cdots,d_{|\theta_i|}}=\THREEC
{l_i^1(1-\frac{d_1}{N})+u_i^1\frac{d_1}{N}}
{\vdots}
{l_i^{|\theta_i|}(1-\frac{d_{|\theta_i|}}{N})+u_i^1\frac{d_{|\theta_i|}}{N}},
\end{align}
where $l_i^j$ and $u_i^j$ are joint limits. In other words, we build a $|\theta_i|$-dimensional grid with $N$ cells along each dimension. Next, we discretize $R_i(\theta_i)$ on the grid and use mixed-integer constraints to ensure that $R_i$ falls inside one of the $N^{|\theta_i|}$ cells. Note that all the $R_i(\theta_i^{d_1,\cdots,d_{|\theta_i|}})$ are precomputed using forward kinematics and used as coefficients of the linear constraints (\prettyref{eq:XChoose}).

Since the palm of the gripper is fixed, we have to inversely transform the target object. As a result, each potential grasp point $p_i$ can be transformed into $Rp_i+t$ where $R\in\SOTHREE$. The technique to relax $\SOTHREE$ as MICP has been presented in \cite{dai2018synthesis} but this technique requires too many binary decision variables. Instead, we use a similar technique to \prettyref{eq:XChoose}. Based on the Rodrigues' formula $R=\E{exp}(w)$, where $w$ is an arbitrary 3D vector, we introduce MICP constraints:
\begin{small}
\begin{equation}
\begin{aligned}
\label{eq:RChoose}
&R=\sum_{d_1,d_2,d_3=1}^N \beta^{d_1,d_2,d_3} 
\E{exp}(\THREEC
{-\pi(1-\frac{d_1}{N})+\pi\frac{d_1}{N}}
{-\pi(1-\frac{d_2}{N})+\pi\frac{d_2}{N}}
{-\pi(1-\frac{d_3}{N})+\pi\frac{d_3}{N}})   \\
&\sum_{d_1,d_2,d_3=1}^N \beta^{d_1,d_2,d_3}=1  \\
&\sum_{d_1,d_2}\beta^{d_1,d_2,d_3},
\sum_{d_1,d_3}\beta^{d_1,d_2,d_3},
\sum_{d_2,d_3}\beta^{d_1,d_2,d_3}\in\SOSTWO,
\end{aligned}
\end{equation}
\end{small}
which requires $3\lceil\E{log}_2N\rceil$ binary decision variables and $\beta^{d_1,d_2,d_3}$ are continuous-valued auxiliary variables. Given these constraints, the requirement that the $i$th fingertip point is at $p_j$ can be formulated as a linear constraint:
\begin{align}
\label{eq:EECons}
R_ix_i+t_i=Rp_j+t.
\end{align}
In summary, we reduce the IK problem for the gripper to a set of linear constraints, whose feasibility can be efficiently verified using off-the-shelf solvers such as \cite{gurobi}. Putting the two parts together, our formulation needs $(|\theta|-3)\lceil\E{log}_2N\rceil$ binary decision variables to solve the IK problem.
\subsection{Low-Level BB\label{sec:low}}
The goal of solving low-level BB is to check whether a $\BBNode(X^1,\cdots,X^K)$ contains a feasible solution in terms of gripper's kinematics. In \prettyref{sec:IK}, the IK problem is formulated as a MICP. However, solving IK is not enough for feasibility checks of $\BBNode{}$s because \prettyref{eq:EECons} constrains that each $x_i$ can only be at one point, while a $\BBNode$ generally allows $x_i$ to be at one of several points in non-leaf cases. In the latter case, we have at least one $|X^i|>1$ so that $x_i$ can be at any point in the set $\{Rp_j+t|p_j\in X^i\}$. In order for the feasibility check to be performed using the off-the-shelf MICP solver \cite{gurobi}, we have to relax this point-in-set constraint as a linear or conic constraint. A typical relaxation is to constrain that $x_i$ lies in the convex hull of the set. However, this constraint takes the following form which is not convex:
\begin{align}
\label{eq:CCons}
R_ix_i+t_i=\sum_jw_j(Rp_j+t)\quad w_j\geq0\quad \sum w_j=1.
\end{align}
This is because $w_j$ and $R$ are both variables, leading to a bilinear form. It is possible to relax a bilinear form into MICP by requiring additional binary decision variables. Instead, we propose to construct a minimal bounding sphere for the set $X^i$ denoted as:
\begin{align*}
X^i\subseteq B(X^i)\triangleq\{x|\|x-c(X^i)\|^2\leq r(X^i)\}, 
\end{align*}
where $c(X^i)$ is the center of the sphere and $r(X^i)$ is the squared radius. Next, we relax the point-in-set constraint as:
\begin{align}
\label{eq:BCons}
\|R_ix_i+t_i-Rc(X^i)-t\|^2\leq r(X^i),
\end{align}
which is a quadratic cone and can be handled by \cite{gurobi}. Note that $c(X^i)$ and $r(X^i)$ are constants and can be precomputed for each node of the KD-tree (see Appendix A for details). A minor issue is that \prettyref{eq:BCons} is not as tight as \prettyref{eq:CCons} in terms of the volume of the constrained space. To alleviate this problem, we notice that if $X^i$ has a parent in KD-tree denoted as $X_p$, then $x_i$ should also satisfy the point-in-set constraint of $X_p$. Therefore, we can backtrace $X^i$ to the root KD-tree node and add all the constraints of \prettyref{eq:BCons} along the path, as illustrated in \prettyref{fig:pipeline}e.
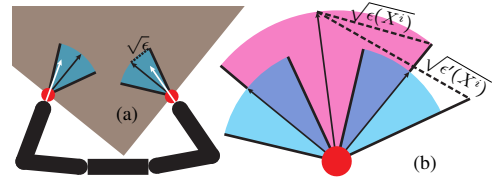
\begin{figure}[ht]
\vspace{-10px}
\centering
\scalebox{0.65}{
\includegraphics[width=0.45\textwidth]{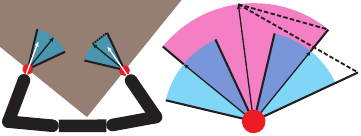}
\put(-180,30){(a)}
\put(-30 ,5 ){(b)}
\put(-175,65){\small$\sqrt{\epsilon}$}
\put(-63 ,83){\rotatebox{-14}{\small$\sqrt{\epsilon(X^i)}$}}
\put(-26 ,59){\rotatebox{-30}{\small$\sqrt{\epsilon'(X^i)}$}}}
\vspace{-10px}
\caption{\label{fig:normal} \small{(a): We illustrate the normal of fingertip points $n(x_i)$ (white arrow) and the inward normal of potential grasp points $n(p_j)$ (black arrow). We allow $n(x_i)$ to lie in a normal cone around $n(p_j)$ (blue) with a threshold $\epsilon$ (dashed line). (b): We illustrate the relaxed normal cone of the two potential grasp points (red) with threshold denoted as $\epsilon(X^i)$. The final threshold used in the constraint is $\epsilon'(X^i)$, taking the user-defined threshold into consideration. All vectors have unit length. We use an extruded red region for clarity.}}
\vspace{-10px}
\end{figure}
\subsection{Normal Constraints}
We can further optimize our formulation by taking the surface normals of the target object into consideration, leading to even tighter constraints. As illustrated in \prettyref{fig:normal}a, each potential grasp point $p_j$ can be associated with an inward surface normal denoted by $n(p_j)$. Also, each fingertip point $x_i$ can also be associated with a normal $n(x_i)$. It is intuitive to constrain that $n(x_i)$ should be pointing at a similar direction to $n(p_j)$. In practice, we do not need $n(x_i)$ to align with $n(p_j)$ exactly but allow $n(x_i)$ to lie in a small vicinity. Therefore, if a leaf $\BBNode(X^1,\cdots,X^K)$ is encountered, then we add the following constraint to MICP for each $X^i=\{p\}$:
\begin{align}
\label{eq:NCons}
\|R_in(x_i)-Rn(p)\|^2\leq\epsilon,
\end{align}
where $\epsilon$ is a user-defined threshold. If a non-leaf $\BBNode$ is encountered, then we have to add a normal-in-set constraint. Using a similar technique as \prettyref{sec:low}, we construct a normal cone denoted as:
\begin{align*}
\{n(p)|p\in X^i\}\subseteq C(X^i)\triangleq\{n|\|n-m(X^i)\|^2\leq \epsilon(X^i)\},
\end{align*}
for each internal KD-tree node during precomputation. Here $m(X^i)$ is the central direction of the normal cone and $\epsilon(X^i)$ is the squared radius. We can then add the relaxed normal-in-set constraint for $X^i$:
\begin{align}
\label{eq:NConsR}
\|R_in(x_i)-Rm(X^i)\|^2\leq\epsilon'(X^i),
\end{align}
where $\epsilon'(X^i)$ is the squared radius of the normal cone taking the user-defined threshold into consideration, as illustrated in \prettyref{fig:normal}b (see Appendix A). Finally, we can further tighten the normal-in-set constraint using a similar technique as \prettyref{sec:low}. We can backtrace $X^i$ to the root KD-tree node and add all constraints of \prettyref{eq:NConsR} along the path.

\subsection{Collision Handling using Lazy-MICP}
In addition to checking the gripper's kinematics feasibility, our low-level BB also ensures that gripper's links do not collide with each other or with the target object. It has been shown in \cite{7076321,dai2018synthesis} that collision constraints can be relaxed as MICP. In order to reduce the use of binary decision variables, we propose to add collision constraints in a lazy manner. 

Specifically, we assume that the target object $\Omega_o$ and all gripper links $\Omega_i$ are convex objects. If $\Omega_o$ is not convex then we can approximate it using a union of convex shapes. We first ignore all collision constraints and solve MICP. We then detect collisions between $R\Omega_o+t$ and $\Omega_i(\theta)$ and record the pair of points with the deepest penetration denoted as $D$, e.g. using \cite{1260767}. If we find that $a\in\Omega_o$ and $b\in\Omega_i$ are in collision, then we pick a separating direction from a set of possible separating directions $\{s_1,\cdots,s_S\}$ and introduce the following constraint as illustrated in \prettyref{fig:collision}a and \prettyref{fig:collision}b:
\begin{small}
\begin{equation}
\begin{aligned}
\label{eq:collA}
&s_k^T(Ra+t)+D\leq s_k^T(R_ib+t_i)+(1-\gamma_k^{oi})M\;\forall k    \\
&\gamma_k^{oi}\geq0
\quad\sum_{k=1}^S\gamma_k^{oi}=1
\quad\gamma_1^{oi},\cdots,\gamma_k^{oi}\in\SOSONE,
\end{aligned}
\end{equation}
\end{small}
where $\SOSONE$ is the special ordered set of type 1 \cite{vielma2011modeling}, $\gamma_k^{oi}$ are the auxiliary variables, and $M$ is the big-M parameter \cite{samuel1972iterative}. Similarly, if there is a collision between a pair of points, $a\in\Omega_i$ and $b\in\Omega_j$, then we have the following constraint:
\begin{small}
\begin{equation}
\begin{aligned}
\label{eq:collB}
&s_k^T(R_ja+t_j)+D\leq s_k^T(R_ib+t_i)+(1-\gamma_k^{ji})M\;\forall k    \\
&\gamma_k^{ji}\geq0
\quad\sum_{k=1}^S\gamma_k^{ji}=1
\quad\gamma_1^{ji},\cdots,\gamma_k^{ji}\in\SOSONE.
\end{aligned}
\end{equation}
\end{small}
After adding collision constraints, the new MICP is solved again with a warm-start and we again perform collision-detection. This is looped until no new collisions are detected or MICP becomes infeasible. Note that if a new collision is detected for a link-link or link-object pair for which collision has been detected in previous loops, then only the first lines of \prettyref{eq:collA} and \prettyref{eq:collB} are needed. In other words, binary variables are needed once for each link-link and link-object pair and the binary variables number is $\lceil\E{log}_2S\rceil$. 

Note that the collisions between the first K fingertip links and the target object do not need to be detected or handled by MICP. This is because each fingertip link contacts the target object at one point with matched normal when \prettyref{eq:NCons} holds with $\epsilon=0$, which is a sufficient condition for two convex objects to be collision-free \cite{rockafellar-1970a}. In practice, we allow users to set a small, positive $\epsilon$ to account for inaccuracies in gripper and target object shapes.
\begin{figure}[ht]
\centering
\scalebox{0.70}{
\includegraphics[width=0.45\textwidth]{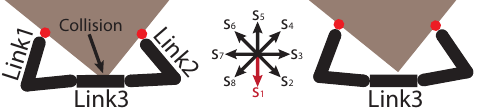}
\put(-235,35){(a)}
\put(-130,45){(b)}
\put(-10 ,35){(c)}}
\vspace{-10px}
\caption{\label{fig:collision} \small{A \lmr{2D} illustration of collision handling algorithm. (a): There is collision between Link3 and the target object. (b): MICP selects one of the 8 possible directions. (c): Collision can be resolved when $s_1$ is selected (red). MICP does not need to consider the collisions between Link1, Link2 and the target object because they contact at one point with matched normal. \lmr{We choose $S$ = 64 in 3D experiments.}}}
\vspace{-10px}
\end{figure}
\section{ALGORITHM ACCELERATION\label{sec:acceleration}}
Our method discussed in \prettyref{sec:formulation} is computationally costly due to the repeated use of the MICP-based IK algorithm, to check the kinematics feasibility of the gripper. In this section, we discuss three techniques to reduce the cost of MICP solving. Our first technique is \TE{bottom-up kinematics check}, which is based on the following observation:
\begin{lemma}
\label{Lem:Feas}
If the MICP-based IK problem for a \BBNode{} is feasible, then the MICP-based IK problem for its parent is also feasible.
\end{lemma}

\begin{proof}
The IK problem for a \BBNode{} is derived by adding more constraints (in forms of \prettyref{eq:EECons},\ref{eq:NCons},\ref{eq:BCons},\ref{eq:NConsR}) to the IK problem of its parent. (See \prettyref{Alg:low} for more details on the construction of a MICP-based IK problem.)
\end{proof}

Therefore, we can check the gripper's kinematics feasibility lazily. Specifically, if a \BBNode{} is a non-leaf node and it has not been checked for gripper's kinematics feasibility, we skip the check and continue branching. If a \BBNode{} is a leaf node, we solve MICPs to check for gripper's kinematics feasibility for all the \BBNode{}s on the path between this leaf node and the root \BBNode{}. If any of the MICP appears to be feasible in this process, all ancestor nodes will also be feasible and their checks can be skipped. Our second technique is \TE{warm-started MICP solve}. We store the solution of MICP for each \BBNode{} and use this solution as the initial guess for the MICP solves of its children. Our third technique is \TE{local optimization}. Note that MICPs are convex relaxations of non-convex optimization problems. Non-convex optimization problems are sub-optimal but efficient to solve. Therefore, we propose to solve a non-convex optimization before invoking a call to the MICP solver. If the non-convex optimization appears to be feasible, we skip the MICP solves. We use interior point algorithm \cite{byrd2006k} as our non-convex optimization solver. The key steps of our algorithm are illustrated in \prettyref{fig:steps} and we summarized our method in Appendix B.
\begin{figure}[ht]
\vspace{-20px}
\centering
\scalebox{0.7}{
\includegraphics[width=0.45\textwidth]{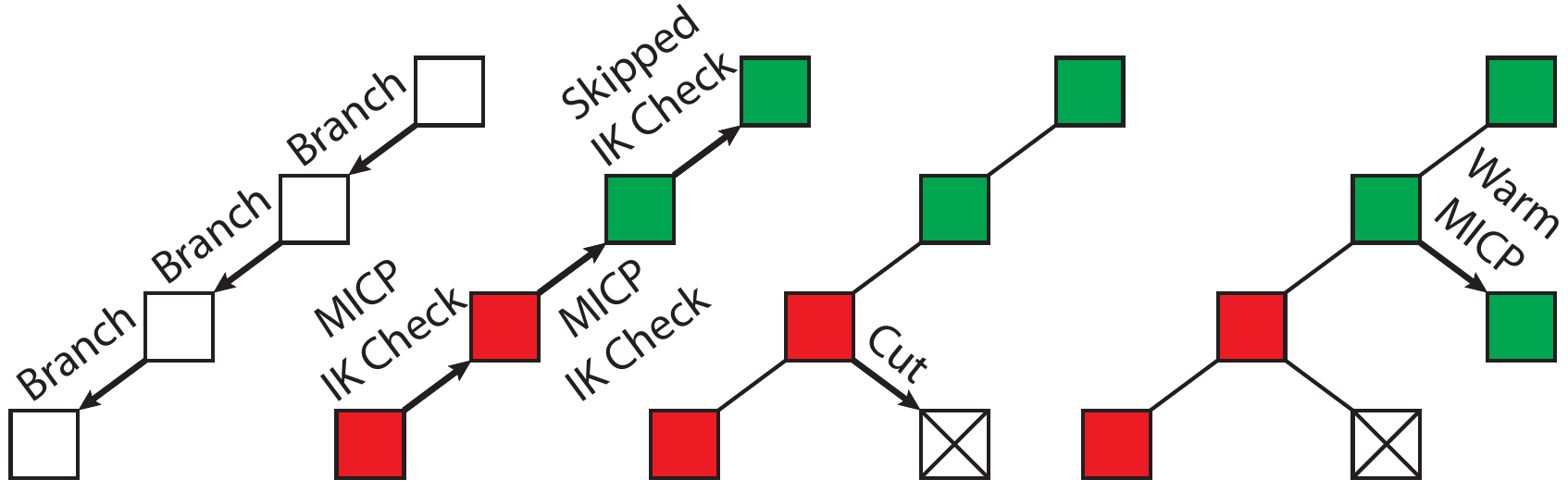}
\put(-212,3){(a)}
\put(-165,3){(b)}
\put(-120,3){(c)}
\put(-57 ,3){(d)}}
\vspace{-10px}
\caption{\label{fig:steps} \small{We illustrate key steps of our algorithm. (a): We skip MICP-based IK checks for non-leaf \BBNode{}s. (b): We solve MICP for \BBNode{}s in a bottom-up manner. If a \BBNode{} is infeasible (red), then the feasibility of its parent \BBNode{} must be checked by solving another MICP. If a \BBNode{} is feasible (green), then its parent must be feasible and we can skip the check. (c): Another leaf \BBNode{} is cut due to the infeasibility  of its parent. (d): The MICP solve on a \BBNode{} can be warm-started from a parent \BBNode{}.}}
\vspace{-15px}
\end{figure}
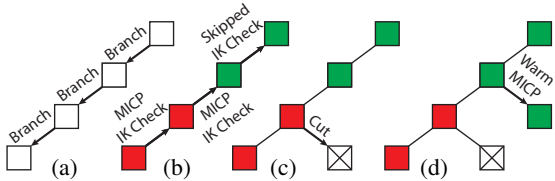

\section{EXPERIMENTS \& RESULTS\label{sec:results}}
We perform all the experiments on a single desktop machine with one Intel I7-8750H CPU (6-cores at $2.2$Hz). Given a target object, we first sample $p_1,\cdots,p_P$ on $\partial\Omega_o$ using parallel Poisson disk sampling \cite{wei2008parallel} and then build a KD-tree for the set of $P$ points using \cite{popov2006experiences}. Finally, we solve low-level MICP problems using \cite{gurobi}. To grasp the object, we use a 3-finger axial-symmetric gripper with \lmr{$|\theta|=6+3\times3=15$} and $|\theta_i|=3$. Each finger of the gripper is controlled by one ball joint and one hinge joint. Under this setting, our IK formulation requires $12\lceil\E{log}_2N\rceil$ binary decision variables while \cite{dai2018synthesis} requires $630\lceil\E{log}_2N\rceil$ binary decision variables. The average solving time using our formulation and \cite{dai2018synthesis} are compared in \prettyref{table:IKCompare}, which indicates that our formulation is over $100\times$ more efficient.
\setlength{\tabcolsep}{8pt}
\begin{table}[ht]
	\vspace{-7px}
	\begin{center}
		\caption{\label{table:IKCompare} \small{Our MICP-based IK formulation is over $100\times$ faster than \cite{dai2018synthesis} because we use very few binary decision variables. From left to right: number of pieces in discretization, \#binary decision variables using our formulation, \#binary decision variables using \cite{dai2018synthesis}, our average solve time, and the average solve time of \cite{dai2018synthesis} (50 random trials).}}
		\vspace{-5px}
		\scalebox{0.8}{
			\begin{tabular}{ccccc}
				\toprule
				$N$ & \#Binary Ours & \#Binary \cite{dai2018synthesis} & Ours(s) & \cite{dai2018synthesis}(s) \\
				\midrule
				2 & 12 & 630  & 0.034s & 23.021s \\
				4 & 24 & 1260 & 1.231s & 287.741s \\
				8 & 36 & 1890 & 48.366s & 8632.237s \\
				\bottomrule
		\end{tabular}}
	\end{center}
	\vspace{-19px}
\end{table} 

\begin{figure*}[th]
	\centering
	\includegraphics[width=0.8\textwidth]{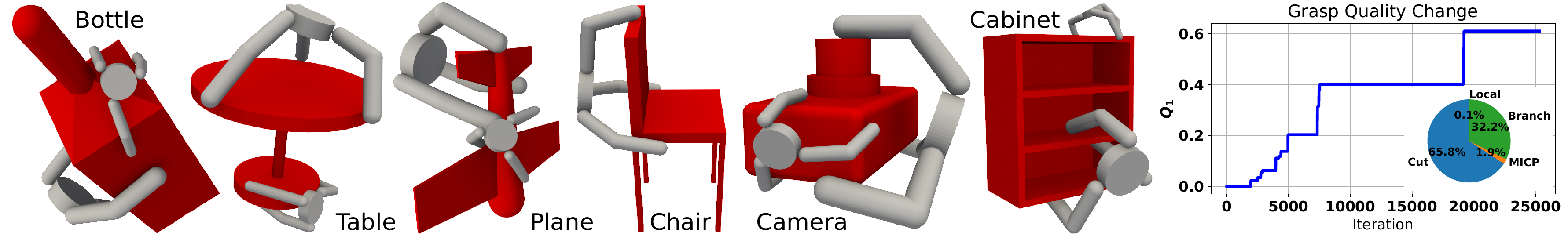}
	\vspace{-5px}
	\caption{\label{fig:mainResults} \small{We show a list of optimal grasps found for the 3-finger axial-symmetric gripper and a row of 6 objects. For some objects we show two different grasps, one for a large gripper and the other for a small gripper. Finally, we plot the convergence history and computational cost of different substeps of our algorithm for the bottle model.}}
\end{figure*}
A list of results is demonstrated in \prettyref{fig:mainResults} and we show the convergence history for one instance. In these examples, we choose $P=100, S=8, N=8, Q=Q_1, \epsilon=0.05$. Under this setting, our algorithm needs to explore $30-60$K \BBNode{}s in order to find the optimal grasp and the computation takes 20-180 minutes depending on the complexity of target object shapes. We also plot the computational cost of different substeps of our algorithm, where 65$\%$ of the \BBNode{}s are cut due to incumbent or gripper's kinematics infeasibility, MICP solves are only needed by 1.9$\%$ of the \BBNode{}s, and local optimization can be used to avoid MICP solves need by 0.1$\%$ of the \BBNode{}s. Finally, if we ignore the low-level BB and only run the high-level BB, our algorithm coincides with \cite{7812687}, which only searches for a set of grasp points. The computation corresponding to high-level BB takes less than $20$ minutes. Therefore, the main bottleneck of our algorithm is the gripper's kinematics check or the low-level BB.

\begin{figure*}[th]
	\vspace{-12px}
	\centering
	\includegraphics[width=0.8\textwidth]{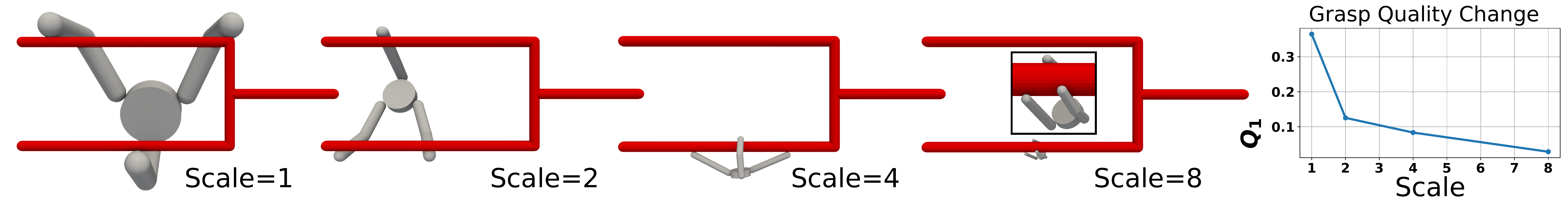}
	\vspace{-5px}
	\caption{\label{fig:scale} \small{We plan grasps with the target object being a tuning fork and rescale the object by $1,2,4,8$ times. As the scale of the object grows, the optimal grasp quality reduces and the gripper can only grasp a smaller part of the target object, leading to lower grasp quality.}}
	\vspace{-10px}
\end{figure*}
In \prettyref{fig:mainResults}, we also show two grasps for some objects using a large and small gripper. The large gripper can hold the entire object. But if the gripper is small, it can only hold a part of the target object. A more systematic evaluation is shown in \prettyref{fig:scale}, where the quality $Q$ monotonically decreases as we use the larger version of the same objects. In \prettyref{table:Time}, we show MICP solving time, total running time and percentage of MICP solving time in total running time.
\setlength{\tabcolsep}{8pt}
\begin{table}[ht]
	\vspace{-20px}
	\begin{center}
		\caption{\label{table:Time} \small{Running time of our results shown in \prettyref{fig:mainResults} with the large gripper.}}
		\vspace{-4px}
		\scalebox{0.8}{
			\begin{tabular}{ccccc}
				\toprule
				\ Object & MICP (min) & Total (min) & Percentage ($\%$) \\
				\midrule
				Bottle & 43.766  & 71.337 & 61.351 \\
				Table & 9.614 & 49.659 & 19.360 \\
				Plane & 5.924 & 34.381 & 17.230 \\
				Chair & 6.004 & 19.669 & 30.525 \\
				Camera & 4.595 & 50.000 & 9.190 \\
				Cabinet & 54.182  & 73.394& 73.823 \\
				\bottomrule
		\end{tabular}}
	\end{center}
	\vspace{-23px}
\end{table}

Finally, we compare the performance of our method with a sampling-based method \prettyref{fig:advantage} using both the 3-finger gripper and the \lmr{10-DOF} Barrett Hand. Being incomplete, a sampling-based method sometimes cannot find solutions, especially when the target object is large compared with the gripper. This is because feasible grasps become rare in the configuration space when object size grows and most samples are not valid. \lmr{For \textit{GraspIt!}\cite{1371616} settings, the space search type is Axis-angle, the energy formulation is AUTO\_GRASP\_QUALITY\_ENERGY, the maximum iteration number is 45000, and the planner type is Sim.Ann.} As a result, the initial guess of the gripper's pose is important when using \cite{1371616}. However, our method can always find a solution when one exists and we do not require users to provide an initial guess. Even when a sampling-based method can find a solution, our solution always has a higher quality in terms of the value of $Q_1$ metric. On the other hand, \cite{1371616} can find a sub-optimal solution within 10min which is $10\times$ faster than our method. \lmr{However, we show in \prettyref{fig:multi_output}a that giving \cite{1371616} more computational time does not improve the solution and we speculate that the solution has fallen into a local minimum.} If only sub-optimal solutions are needed, the user can choose to terminate our algorithm when $Q$ is larger than a threshold. According to the convergence history in \prettyref{fig:mainResults}, our method can usually find feasible solutions after exploring $1-5$K nodes, which also takes several minutes. \lmr{However, if our algorithm is allowed to explore more nodes, as shown in \prettyref{fig:multi_output}b, it can output multiple grasps for an object by storing all the feasible solutions. This makes our algorithm potentially useful for offline grasp dataset construction and online learning-based grasp systems such as \cite{8968115,lu2020planning}.}

\begin{figure}[th]
	\vspace{-5px}
	\centering
	\scalebox{0.8}{
		\includegraphics[width=0.22\textwidth]{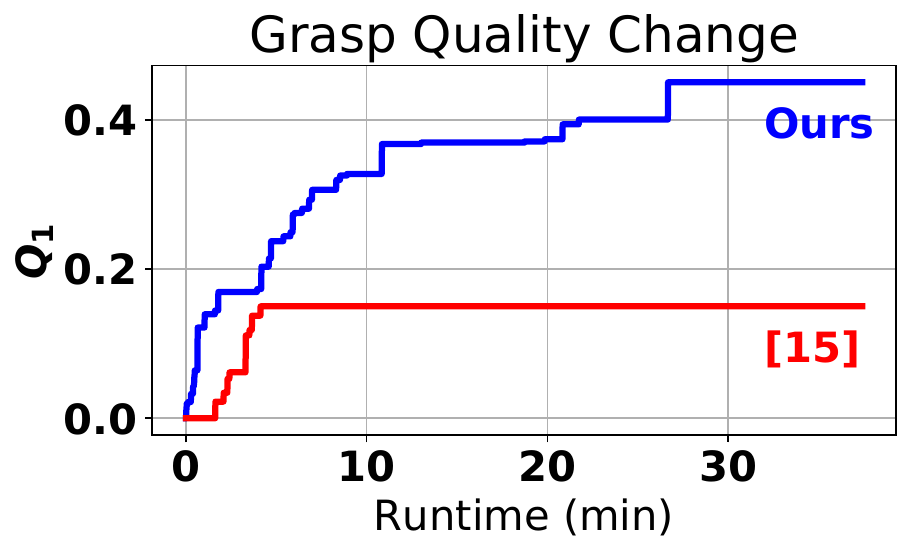}
		\includegraphics[width=0.22\textwidth]{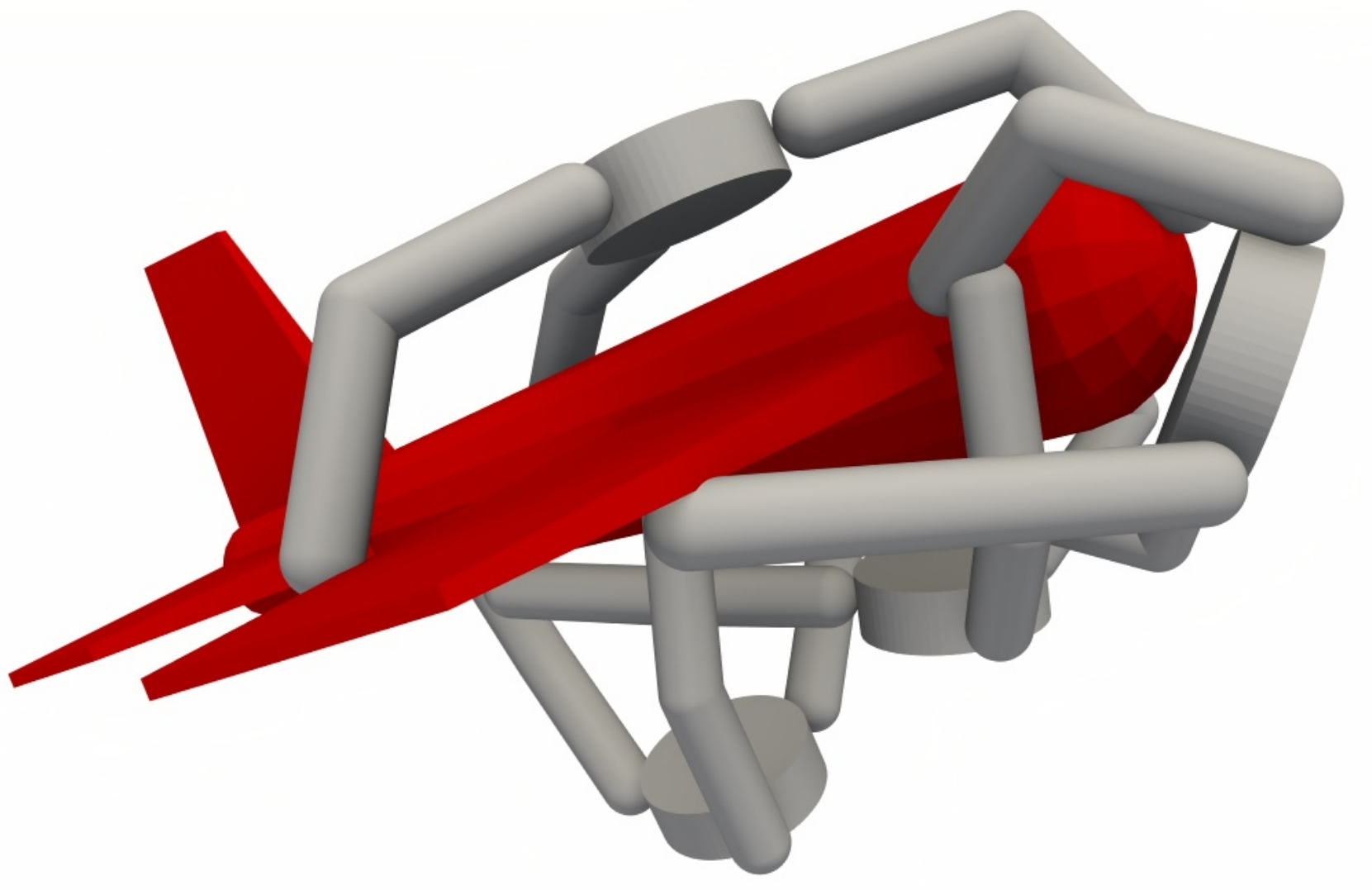}
		\put(-230,5){(a)}
		\put(-10,5){(b)}}
	\vspace{-5px}
	\caption{\label{fig:multi_output} \small{\lmr{(a): $Q_1$ plotted against runtime for \cite{1371616} and our method with the Barrett Hand grasping a bulb (as shown in \prettyref{fig:advantage}). (b): Generating 4 grasps for the plane. }}}
	\vspace{-3px}
\end{figure}
\section{CONCLUSION \& LIMITATIONS}
We present a two-level BB algorithm to search for the optimal grasp pose in a restricted search space that maximizes a given monotonic grasp metric. The high-level BB selects grasp points that maximize grasp quality, while the low-level BB cut infeasible $\BBNode{}$s in terms of gripper's kinematics. Our low-level BB uses a compact MICP formulation that requires a small number of binary variables. Experiments show that our method can plan grasps for complex objects.
\begin{figure}[th]
	\centering
	\vspace{-5px}
	\scalebox{0.8}{
		\includegraphics[width=0.45\textwidth]{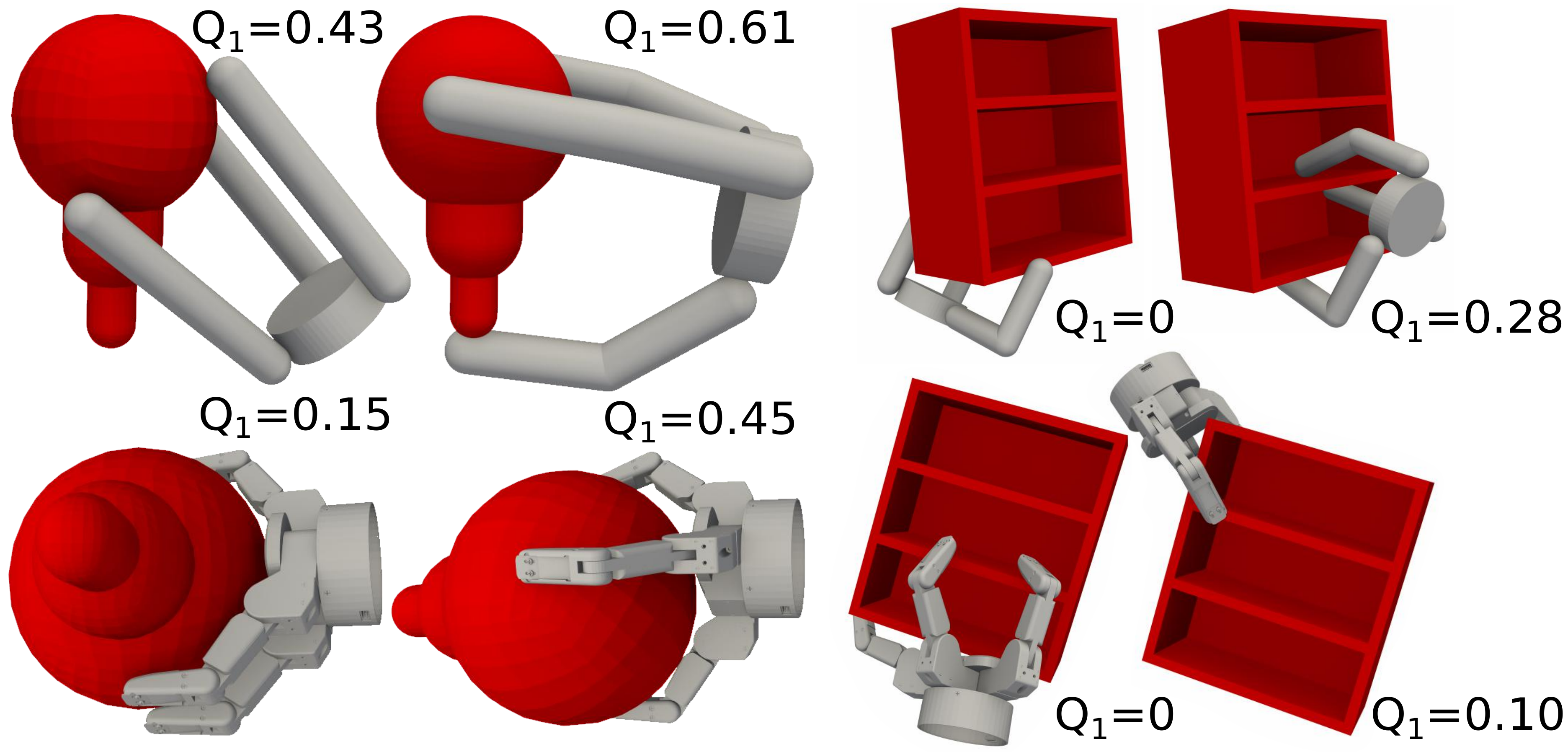}
		\put(-220,-10){(a): \cite{1371616}}
		\put(-160,-10){Ours}
		\put(-100,-10){(b): \cite{1371616}}
		\put(-40 ,-10){Ours}}
	\vspace{-4px}
	\caption{\label{fig:advantage} \small{We show the advantage of our method over the sampling-based method \cite{1371616}. (a): The sampling-based method can mostly find a solution when an object is small, though the grasp quality is less than our solution. (b): When the object is large, the sampling-based method sometimes cannot find a solution, while our method always finds solutions when one exists.}}
	\vspace{-10px}
\end{figure}
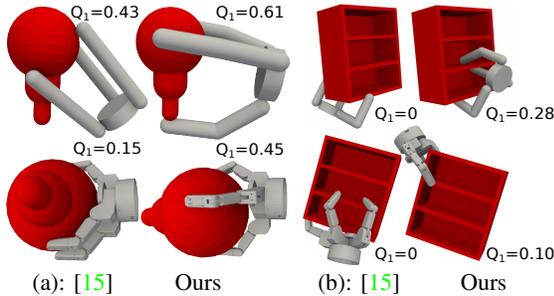

Our work has several limitations. First, we only plan gripper poses without considering other sources of infeasibility such as environmental objects and robot arms. When robot arms are considered, the decoupled assumption of \prettyref{sec:IK} is violated and we need new techniques for relaxing IK as MICP. Second, although our IK relaxation is more efficient than \cite{dai2018synthesis}, our method is not an outer-approximation. In other words, if a gripper pose is feasible using exact IK, it might not be feasible under our relaxed IK constraints. \lmr{Third, we only plan for precision grasps with fingertip-contacts, while generating power grasps or caging grasps is a good topic for future work.} In addition, our formulation incurs a high computational cost for complex object shapes, such as those acquired from scanning real-world objects. \lmr{Finally, our method does not consider modeling or sensing uncertainty, which is necessary to realize the grasp in a physical platform.}
\section*{APPENDIX A: BOUNDING SPHERES \& CONES\label{appen:bound}}

Each KD-tree node $X$, contains a set of potential grasp points $p_j$, each of which has an outward normal direction $n(p_j)$. To generate an efficient MICP-based IK problem, we need to compute a minimal bounding sphere that encloses $p_j$ and a minimal bounding cone that encloses $n(p_j)$. In this section, we present methods to compute these bounding volumes via numerical optimization using the following lemmas: 
\begin{lemma}
The minimal bounding sphere $B(X)$ can be computed by solving the following conic programming problem: 
\begin{small}
\begin{align*}
\argmin{c(X),r(X)} r(X) \quad \E{s.t.} \|c(X)-p_j\|^2\leq r(X)\;\forall p_j\in X,
\end{align*}
\end{small}
where $c(X)$ is the center of the bounding sphere and $r(X)$ is the radius of the bounding sphere.
\end{lemma}
\begin{proof}
Any valid bounding sphere should contain all the potential grasp points $p_j$, which justifies our constraints. A minimal bounding sphere has the smallest radius, which justifies our objective function.
\end{proof}
\begin{lemma}
The minimal bounding cone $C(X)$ can be computed by solving the following non-convex programming problem:
\begin{small}
\begin{align*}
\argmin{\|m(X)\|=1,\epsilon(X)} \epsilon(X) \quad \E{s.t.} \|m(X)-n(p_j)\|^2\leq \epsilon(X)\;\forall p_j\in X,
\end{align*}
\end{small}
where $m(X)$ is the central axis of the normal cone and $\epsilon(X)$ is the radius as defined in \prettyref{fig:normal}.
\end{lemma}
\begin{proof}
Any valid bounding cone should contain all the potential grasp normals $n(p_j)$, which justifies our constraints. A minimal bounding cone has the smallest radius, which justifies our objective function.
\end{proof}
This optimization is non-convex due to the unit length constraint $\|m(X)\|=1$.  To compute the minimal bounding cone, we relax the unit length constraint using MICP via the technique presented in \cite{dai2018synthesis}. Finally, we take the user-defined threshold into consideration and compute $\epsilon'(X)$ as follows:
\begin{small}
\begin{align*}
&\theta\triangleq
2\E{sin}^{-1}(\frac{\sqrt{\epsilon(X^i)}}{2})+
2\E{sin}^{-1}(\frac{\sqrt{\epsilon}}{2})  \\
&\epsilon'(X^i)=\left[2\E{sin}(\frac{\E{min}(\theta,\pi)}{2})\right]^2.
\end{align*}
\end{small}

\section*{APPENDIX B: ALGORITHMS\label{appen:algor}}
A summary of algorithms. Given a gripper and a target object, we first perform a precomputation using \prettyref{Alg:pre}. Afterward, we use \prettyref{Alg:high} as high-level BB and use \prettyref{Alg:low} as low-level BB. The accelerated bottom-up kinematics check is summarized in \prettyref{Alg:BUCheck}, which is used as a middleware between the two levels.
\vspace{-8px}
\begin{algorithm}[ht]
\caption{\label{Alg:pre} Precomputation}
\begin{small}
\begin{algorithmic}[1]
\State Sample $p_1,\cdots,p_P$ on $\partial\Omega_o$ using \cite{wei2008parallel}
\State Construct KD-tree for $p_1,\cdots,p_P$ using \cite{popov2006experiences}
\For{Each $X$ in KD-tree}
\State Construct $B(X)$ and $C(X)$ (Appendix A)
\EndFor
\For{Each link on the gripper}
\State Construct relaxed IK constraints (\prettyref{eq:XChoose})
\EndFor
\State Construct relaxed IK constraint for $\Omega_o$ (\prettyref{eq:RChoose})
\end{algorithmic}
\end{small}
\end{algorithm}

\begin{algorithm}[ht]
\caption{\label{Alg:BUCheck} BottomUpKinematicsCheck($\BBNode$)}
\begin{small}
\begin{algorithmic}[1]
\LineComment{Check if ancestor \BBNode{}.LowLevel=False}
\State Current\BBNode{}$\gets$\BBNode{}
\Do
\If{Current\BBNode{}.LowLevel=False}
\State \BBNode{}.LowLevel=False
\State Return
\EndIf
\State Current\BBNode{}$\gets$Current\BBNode{}.Parent
\doWhile{Current\BBNode{} is not the root}
\State Current\BBNode{}$\gets$\BBNode{}\Comment{Bottom-up Kinematics Check}
\State ChildFeasible$\gets$False
\Do
\If{Current\BBNode{}.LowLevel$\neq$Unknown}
\State Break\Comment{The \BBNode{}/ancestors have been checked}
\ElsIf{ChildFeasible=True}
\State Current\BBNode{}.LowLevel=True
\Else
\State Current\BBNode{}.LowLevel=
\State \quad LowLevelBB(Current\BBNode{})
\EndIf
\State ChildFeasible$\gets$Current\BBNode{}.LowLevel
\State Current\BBNode{}$\gets$Current\BBNode{}.Parent
\doWhile{Current\BBNode{} is not root}
\end{algorithmic}
\end{small}
\end{algorithm}
\vspace{-10px}

\begin{algorithm}[ht]
\caption{\label{Alg:high} HighLevelBB}
\begin{small}
\begin{algorithmic}[1]
\LineComment{\BBNode{}.LowLevel marks gripper's feasibility}
\State $Q_{best}\gets0$ \Comment{\BBNode{}.LowLevel initializes to Unknown}
\State $\E{Queue}\gets\emptyset$, $\E{Queue}$.insert($\BBNode(X_R,\cdots,X_R)$)
\While{$\E{Queue}\neq\emptyset$}
\State $\BBNode(X^1,\cdots,X^K)\gets\E{Queue}$.pop()
\If{$\BBNode$ is leaf}
\State BottomUpKinematicsCheck(\BBNode{})
\If{\BBNode{}.LowLevel=True}
\State $Q_{curr}\gets Q(X^1\cup X^2\cup\cdots X^K)$\Comment{Bound}
\If{$Q_{curr}>Q_{best}$}
\State $Q_{best}\gets Q_{curr}$
\EndIf
\EndIf
\ElsIf{\BBNode{}.LowLevel=True$\vee$Unknown}
\State Find $|X^i|\geq1$\Comment{Branch}
\State $\E{Queue}$.insert($\BBNode(\cdots,X^i_l,\cdots)$)
\State $\E{Queue}$.insert($\BBNode(\cdots,X^i_r,\cdots)$)
\EndIf
\EndWhile
\State Return $Q_{best}$
\end{algorithmic}
\end{small}
\end{algorithm}

\setlength{\textfloatsep}{2pt}
\begin{algorithm}[ht]
\caption{\label{Alg:low} LowLevelBB($\BBNode(X^1,\cdots,X^K)$)}
\begin{small}
\begin{algorithmic}[1]
\State $\MIP\gets\emptyset$
\State $\MIP$.add(\prettyref{eq:XChoose},\ref{eq:RChoose})
\For{$i=1,\cdots,K$}
\For{Each $X$ from $X^i$ to $X_R$}
\If{$|X^i|=1$}
\State $\MIP$.add(\prettyref{eq:EECons},\ref{eq:NCons}) for $i$
\Else
\State $\MIP$.add(\prettyref{eq:BCons},\ref{eq:NConsR}) for $X$
\EndIf
\EndFor
\EndFor
\Do
\State Solve $\MIP$ using \cite{gurobi}\Comment{Warm-start if possible}
\If{$\MIP$.Feasible=False}
\State Return False
\EndIf
\State Detect penetration $D$ among $\Omega_o$ and $\Omega_i$
\If{$D>0$}
\State $\MIP$.add(\prettyref{eq:collA} or \prettyref{eq:collB})
\EndIf
\doWhile{$D>0$}
\State Return True
\end{algorithmic}
\end{small}
\end{algorithm}

\section*{ACKNOWLEDGMENT}
This work was supported in part by the National Key Research and Development Program of China (No.2018AAA0102200), NSFC (61572507, 61532003, 61622212), ARO grant W911NF-18-1-0313, and Intel. Min Liu is supported by the China Scholarship Council.


\bibliographystyle{IEEEtran}
\bibliography{reference}

\end{document}